\documentclass[11pt,letterpaper]{article}
\usepackage[in]{fullpage} 
\usepackage{float}

\usepackage[utf8]{inputenc}
\usepackage[T1]{fontenc}
\usepackage[english]{babel}
\usepackage{hyperref}       
\usepackage{url}            
\usepackage{booktabs}       
\usepackage{amsfonts}       
\usepackage{nicefrac}       
\usepackage{microtype}      
\usepackage{amsmath, amsthm, amssymb}
\usepackage{fancybox}
\usepackage{graphicx}
\usepackage{float}
\usepackage{algorithm}
\usepackage{color}
\usepackage{tikz}
\usepackage{pgfplots}
\usepackage{array}
\usepackage{subcaption}
\usepackage[textsize=scriptsize]{todonotes}
\usepackage{enumitem}
\usepackage{wrapfig}
\usepackage{enumitem}
\usepackage{etoolbox}
\usepackage{diagbox}
\usepackage{comment}
\usepackage{authblk}
\usepackage{mathpazo}

\newtheoremstyle{slplain}
  {.4\baselineskip\@plus.1\baselineskip\@minus.1\baselineskip}
  {.3\baselineskip\@plus.1\baselineskip\@minus.1\baselineskip}
  {\itshape}
  {}
  {\bfseries}
  {.}
  { }
  {}
\theoremstyle{slplain} 

\newtheorem*{definition*}{Definition}
\newtheorem*{theorem*}{Theorem}
\newtheorem{theorem}{Theorem}[section]
\newtheorem{lemma}[theorem]{Lemma}

\makeatletter
\newtheorem*{rep@theorem}{\rep@title}
\newcommand{\newreptheorem}[2]{%
\newenvironment{rep#1}[1]{%
 \def\rep@title{#2 \ref{##1}}%
 \begin{rep@theorem}}%
 {\end{rep@theorem}}}
\makeatother

\newreptheorem{theorem}{Theorem}
\newreptheorem{lemma}{Lemma}
\newreptheorem{claim}{Claim}
\newreptheorem{corollary}{Corollary}
\newreptheorem{proposition}{Proposition}

\theoremstyle{definition}

\theoremstyle{plain} 

\numberwithin{equation}{section}

\newtheoremstyle{etplain}
  {.0\baselineskip\@plus.1\baselineskip\@minus.1\baselineskip}
  {.0\baselineskip\@plus.1\baselineskip\@minus.1\baselineskip}
  {\itshape}
  {}
  {\bfseries}
  {.}
  { }
  {}

\def\relu{\mathrm{relu}}
\def\diag{\mathrm{diag}}
\def\tr{\mathrm{tr}}

\usepackage[
backend=biber,
style=numeric
]{biblatex}

\title{Training Safe Neural Networks with Global SDP Bounds}
\author[1]{Roman Soletskyi\thanks{roman.soletskyi@ens.psl.eu}}
\author[2]{David “davidad” Dalrymple\thanks{david.dalrymple@aria.org.uk}}
\affil[1]{École Normale Supérieure (ENS)–PSL}
\affil[2]{UK Advanced Research + Invention Agency}
\date{September 15, 2024}

\begin{document}

\maketitle

\begin{abstract}
    This paper presents a novel approach to training neural networks with formal safety guarantees using semidefinite programming (SDP) for verification. Our method focuses on verifying safety over large, high-dimensional input regions, addressing limitations of existing techniques that focus on adversarial robustness bounds. We introduce an ADMM-based training scheme for an accurate neural network classifier on the Adversarial Spheres dataset, achieving provably perfect recall with input dimensions up to $d=40$. This work advances the development of reliable neural network verification methods for high-dimensional systems, with potential applications in safe RL policies.
\end{abstract}

\section{Introduction}
\label{sec:intro}

Reinforcement learning (RL) has demonstrated remarkable capabilities in tackling complex control problems, leading to a growing interest in its application to safety-critical domains. Nevertheless, traditional RL algorithms prioritize optimizing performance without incorporating safety constraints, which poses significant risks in environments where unsafe actions can result in severe consequences. Therefore, it is both a challenging and practical problem to develop a neural RL policy along with a formal certificate guaranteeing its safety.

To characterize the safety and desired behavior of a system moving in space, one introduces safe and target subsets, respectively. The objective is to produce a policy that always remains within the safe subset, given a proper starting point, and eventually reaches the target subset. One line of work~\cite{Hu2020},\cite{Everett2021} focuses on over-approximating forward reachable sets; if they always lie within the safe region, one can immediately deduce that the system is safe. Another line of research, notably\cite{Chang2022}, produces a neural Lyapunov function that constrains the dynamics to remain within the attraction region, while~\cite{Zikelic2022} generalizes the Lyapunov function to a reach-avoid supermartingale in stochastic systems.

A major bottleneck of the described approaches is that one typically has to verify that the closed neural system satisfies a property over the entire safe region. For instance, in~\cite{Chang2022}, one must check that the neural Lyapunov function does not increase under the neural policy at any point in the state space. Even if the dynamics are linear, the non-linearity of the neural networks involved in the property statement makes verification challenging, especially when the state space is high-dimensional. There has been significant progress in the neural network verification field, as evidenced by the VNN-COMP contest~\cite{brix2023fourthinternationalverificationneural}. The main goal of obtaining tight bounds on a network's output has been achieved through various methods such as symbolic interval propagation~\cite{wang2018formalsecurityanalysisneural}, linear perturbation analysis~\cite{xu2020automaticperturbationanalysisscalable}, linear programming (LP) and mixed-integer programming (MIP) formulations~\cite{bunel2018unifiedviewpiecewiselinear}, and branch-and-bound (BaB)~\cite{ferrari2022completeverificationmultineuronrelaxation}. Alternatively, preimage approximation was explored in the works of~\cite{kotha2024provablyboundingneuralnetwork} and~\cite{zhang2024premapunifyingpreimageapproximation}.

However, it is worth noting that the primary focus of these works and benchmarks has been provable adversarial robustness~\cite{fan2021adversarialtrainingprovablerobustness}. In this setting, one considers small perturbations around dataset points and checks whether the network's predictions remain robust under these perturbations. Based on the VNN-COMP 2021-2023 winning method, $\alpha,\beta$-CROWN~\cite{zhang2018efficientneuralnetworkrobustness},\cite{xu2021fastcompleteenablingcomplete},\cite{wang2021betacrownefficientboundpropagation}, the best strategy is to combine cheap but tight linearly propagated bounds with BaB. We argue that in the safety control setting, the safe region is too large and cannot be treated as a small perturbation around its center. As a result, the bounds used are too crude, and branching becomes exponentially costly as the system's dimension increases. This problem is fundamental, and the bounds remain loose even when replaced with more expensive LP methods, a result proven theoretically in Section~\ref{subsec:theory}.

As a remedy, we propose using much tighter, though more expensive, semidefinite programming (SDP) for neural network verification~\cite{Raghunathan2018}. To reduce the number of details and focus on the fundamental problem of producing tight bounds for large, high-dimensional input regions, we consider a shallow neural network classifier on the Adversarial Spheres dataset~\cite{gilmer2018adversarialspheres} and, as a safety specification, require 100\% recall on one of the classes. We introduce a novel ADMM-based method to train a classifier neural network that satisfies this safety specification and successfully train a network with an input dimension of $d=40$, which is much larger than the typical dimensions in papers on provably safe control. For instance, the largest system considered in~\cite{Hu2020} is a 6D quadrotor model~\cite{Manzanas2019}.

\section{Preliminaries} 
As discussed in Section~\ref{sec:intro}, when the input region becomes sufficiently large, empirically, linear-based bounds become too loose compared to SDP-based ones, and verification over the entire region becomes prohibitively costly. The main purpose of this section is to conduct a series of experiments to quantify this effect. We describe the paper's setup, review how SDP bounds are computed, discuss the experiment results, and provide a theoretical explanation.

\subsection{Setup}
\label{subsec:setup}
We are interested in verifying a 2-class neural network classifier $f_\theta: \mathbb{R}^d\to\mathbb{R}$ that satisfies a simple provable safety specification. Let $X$ be a convex region; we aim to efficiently verify the bound
\begin{equation}
    \label{eq:setup:bound}
    \max_{x\in X} f_\theta(x)\leq B
\end{equation}
For example, if one class lies entirely within $X$, verifying this condition gives us a classifier that is always correct for the elements of this class, assuming $B$ is taken as the threshold value. To further simplify, we define
\begin{equation}
    \label{eq:setup:input}
    X = \{x: ||x||_p\leq\varepsilon\}
\end{equation} 
where we only consider $p=2$ or $p=\infty$. 

We assume that the neural network $f_\theta$ has a feedforward structure with $L$ layers and ReLU-only activations $\sigma$. We take hidden size $h$ to be uniform across all layers. The network's output is computed recursively by computing its post-activations $x_l$ as follows
\begin{align}
    \label{eq:setup:input}
    & x_0 = x \\
    \label{eq:setup:post-activ}
    & x_l = \sigma(W_l x_{l-1} + b_l),\ l = 1,\dots,L \\
    \label{eq:setup:output}
    & f_\theta(x) = W_{L+1} x_L + b_{L+1}
\end{align}

As the dataset $\mathcal{D}_2$ that the classifier $f_\theta$ is trained on, we use the Adversarial Spheres dataset~\cite{gilmer2018adversarialspheres}. This is a synthetic dataset consisting of two concentric $d$-dimensional spheres with radii normalized to 1 and $R=1.3$. Formally, a new point $x$ and its label $y$ are sampled from $\mathcal{D}_2$ as follows
\begin{align}
    & y\sim\mathrm{Bernoulli}(1/2); & & r = R\cdot y + 1\cdot(1-y); \\
    & z \sim \mathcal{N}(0, I_d); & & x = r\frac{z}{||z||_2};
\end{align}
The paper~\cite{gilmer2018adversarialspheres} reveals that even when a classifier is trained to achieve very high accuracy, it is still easy to find adversarial examples. Consequently, when trained in a regular manner, we expect that property~\eqref{eq:setup:bound} will not hold due to adversarial points $x\in X$, making it challenging to synthesize a provably safe classifier. In experiments we also consider dataset $D_\infty$, where instead of concentric spheres, we use the surface of an $L_\infty$-box with size $2$ and $2R$. 

\subsection{SDP bound computation}
\label{subsec:sdp}
This subsection introduces SDP bound for~\ref{eq:setup:bound} and we focus on computing it precisely and efficiently in the following subsections. We follow~\cite{Raghunathan2018} by introducing a quadratically constrained quadratic program (QCQP) and then relaxing it to SDP. The key idea is that computing $z = \relu(x)$ is equivalent to constraining $z$ with
\begin{equation}
    z(z - x) = 0;\quad z\geq 0;\quad z\geq x
\end{equation}

We use $\odot$ to denote the pointwise multiplication of vectors' entries. Then, instead of computing post-activations~\ref{eq:setup:post-activ} we can constrain them with
\begin{align}
    \label{eq:sdp:pos-activ-constr}
    & x_l \geq 0 \\
    & x_l \geq W_l x_{l-1} + b_l \\
    & x_l\odot\left(x_l - W_l x_{l-1} - b_l\right) = 0
\end{align}

Let us gather all activations into one vector $v$ and introduce matrix $P$ as 
\begin{equation}
    \label{eq:sdp:master-matrix}
    v = \begin{bmatrix}
        & 1 \\
        & x_0 \\
        & x_1 \\
        & \ldots \\
        & x_L
    \end{bmatrix};\quad P = vv^T = \begin{bmatrix}
        & 1 & x_0^T & \ldots & x_L^T \\
        & x_0 & \ldots & \ldots & \ldots \\
        & \ldots & \ldots & x_l x_m^T & \ldots \\
        & x_L & \ldots & \ldots &x_L x_L^T 
    \end{bmatrix}
\end{equation}
Constraints~\ref{eq:sdp:pos-activ-constr} can be rewritten in matrix form, where by $P[x_lx_m^T]$ we denote matrix block with entries $x_l x_m^T$ and by analogy for $P[x_l]$ and $P[1]$, as

\begin{align}
    \label{eq:sdp:constr-pos}
    & P[x_l] \geq 0 \\
    & P[x_l] \geq W_l P[x_{l-1}] + b_l \\
    \label{eq:sdp:constr-block}
    & \diag(P[x_lx_l^T]) = \diag(W_l P[x_{l-1}x_l^T]) + b_l\odot P[x_l]
\end{align}

Finally, we need to define the constraints on the neural network's inputs and the objective to optimize. If we take $p=2$ in the input region $X$ as defined in~\ref{eq:setup:input}, the input constraint is
\begin{equation}
    ||x_0||_2^2 = \tr x_0x_0^T = \tr P[x_0x_0^T] \leq\varepsilon^2
\end{equation}
If we take $p=\infty$, the input constraint is
\begin{equation}
    |x_0|\leq\varepsilon \iff x_0 \odot x_0 = \diag x_0x_0^T = \diag P[x_0x_0^T]\leq\varepsilon^2
\end{equation}
According to~\ref{eq:setup:bound}, we maximize under the above constraints
\begin{equation}
    B^{\mathrm{QCQP}} = \max W_{L+1}x_L + b_{L+1} = \max W_{L+1}P[x_L] + b_{L+1}
\end{equation}

The constructed QCQP is an NP-hard problem, but this matrix form provides an amenable SDP relaxation. Note that $P$ is a rank-1 semidefinite matrix. Therefore, we add additional constraints $P\succeq 0$, $P[1] = 1$, and relax $P$ to be an arbitrary rank semidefinite matrix, yielding the bound $B^{\mathrm{SDP}}$. 

By introducing slack variables, it's possible to put the SDP relaxation into the canonical form that we use further on
\begin{align}
    \label{eq:sdp:primal}
    & B^{\mathrm{SDP}} = \max_X \langle C, X\rangle + c \\
    & \mathcal{A}(X) = a \\
    & X\succeq 0
\end{align}
where $\langle X, Y\rangle = \tr XY^T = \sum_{ij} X_{ij} Y_{ij}$ is a Frobenius inner product. Further on, we denote by $n$ the size of matrices in the SDP problem and by $m$ the total number of constraints. Consider the space of symmetric $n\times n$ matrices $S^n$ and its semidefinite subset $S^n_+$. Then $C\in S^n$, $X\in S^n_+$ and $a\in\mathbb{R}^n$. The operator $\mathcal{A}: S^n\to\mathbb{R}^m $ is defined by its batch of matrices $A_\alpha$ and acts as $\mathcal{A}(X) = \{\langle A_\alpha, X\rangle\}_{\alpha=1}^m$.

\subsection{Gap experiments}
\label{subsec:gap-exp}
The metric we aim to measure is the gap $\Delta$ between the true maximum value and the bound, defined as
\begin{equation}
    \Delta = B - \max_{x\in X} f_\theta(x)
\end{equation}
The bound $B$ is calculated using SDP, as explained in subsection~\ref{subsec:sdp} or using $\alpha$-CROWN~\cite{xu2021fastcompleteenablingcomplete}, which serves as a representative of linear-based bounds. We chose $\alpha$-CROWN because not many verification methods natively support $l_2$-norm input constraint. For instance, Marabou~\cite{Katz2019TheMF}, MN-BaB~\cite{ferrari2022completeverificationmultineuronrelaxation}, and PyRAT~\cite{pyratanalyzer} focus exclusively on $l_\infty$-norm constraints. To nonetheless assess the verification ability of these methods, we generated random neural networks of various sizes and asked the methods to prove a bound at least as good as $B^{\mathrm{SDP}}$ on the region $\|x\|_\infty\leq 1$. We found that all methods, including $\alpha,\beta$-CROWN, began to struggle at moderate sizes $d, h = 10, 20$. Thus, we conjecture that this is a fundamental limitation of all methods using cheap linear bounds and branching. Changing $\alpha$-CROWN for another linear bound is unlikely to alter the qualitative results of our experiments.

Since calculating $\max_{x\in X} f_\theta(x)$ exactly becomes infeasible for networks beyond a certain size, we use a PGD (Projected Gradient Descent) attack to obtain a lower bound on $\max_{x\in X} f_\theta(x)$. We find that for small networks, the PGD bound and the exact value are very close, up to the stopping tolerance chosen during PGD. Therefore, instead of measuring $\Delta$, we measure its upper bound
\begin{equation}
    \Delta\leq\Delta^{\mathrm{emp}} = B^{\alpha-\mathrm{CROWN},\ \mathrm{SDP}} - B^{\mathrm{PGD}}
\end{equation}

In our experiments, we consider both $p=2$ and $p=\infty$ norms with $\varepsilon = 1$. The neural networks considered are either randomly generated or trained on the dataset $\mathcal{D}_p$ until a threshold accuracy is reached. We vary the input dimension $d$, while keeping the hidden dimension $h = 3d$. 

\begin{figure}[H]
    \centering
    \includegraphics[width=0.8\linewidth]{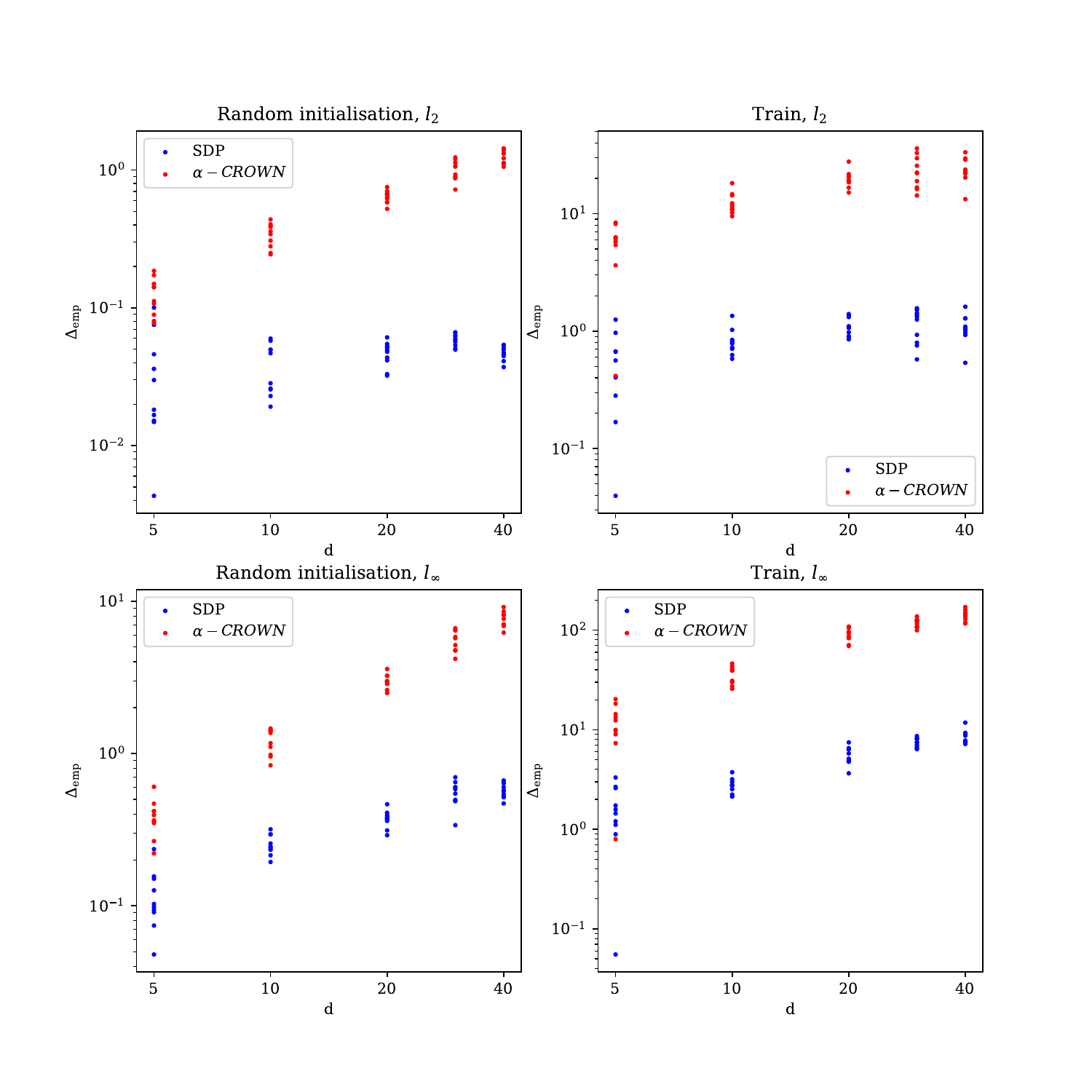}
    \caption{Bound gap $\Delta^{\mathrm{emp}}$ for SDP and $\alpha$-CROWN on $l_{2,\infty}$ sphere for randomly initialized or trained neural networks}
    \label{fig:bound-gap}
\end{figure}

We observe that in almost all cases, for nearly all dimensions, the $\alpha$-CROWN bound is an order of magnitude worse than the SDP bound. Moreover, there is an approximate polynomial scaling of the gap depending on the overall size of the network, as seen in logarithmic plots. On the $l_2$ sphere, the SDP gap appears to hardly grow with increasing $d$, while on the $l_\infty$ norm, it grows sublinearly. When the network is trained, both gaps become much wider compared with randomly initialized setting.

\subsection{Theoretical explanation}
\label{subsec:theory}
We adapt Proposition 1 from~\cite{Raghunathan2018} to our setting and obtain a theoretical upper bound on $B^{\mathrm{SDP}}$ for a randomly initialized neural network. 

\begin{theorem}
    \label{thm:theory:sdp-bound}
    Consider the limit as $d, h\to\infty$ while $d/h$ converges to a constant ratio. Let $f_\theta$ be a Xavier initialized~\cite{glorot2010} feedforward neural network with $L$-hidden layers and the structure described in subsection~\ref{subsec:setup}. Then, for a given $p=2$ or $p=\infty$ and $\varepsilon$, the SDP bound over the region $\|x\|_p\leq\varepsilon$ almost surely satisfies
    \begin{enumerate}
        \item If $p=2$, $B^{\mathrm{SDP}} = O(1 + \varepsilon)$
        \item If $p=\infty$, $B^{\mathrm{SDP}} = O(\varepsilon\sqrt{d})$
    \end{enumerate}
\end{theorem}

We prove this bound in the limit of very wide neural networks to simplify the proof (Section~\ref{appendix:theory:sdp}), as this approach relies on known results about the concentration of norms of random matrices. However, the result should still hold for finite dimensions ($d, h$) and is practically relevant, as demonstrated - at least qualitatively - by the experiments in the previous section. Furthermore, we prove that linear methods, such as linear programming and $\alpha$-CROWN, produce loose bounds even when the network has only one hidden layer.

\begin{theorem}
    \label{thm:theory:alpha-bound}
    Consider the limit $d, h\to\infty$ while $d / h$ converges to a constant ratio. Let $f_\theta$ be a Xavier initialized, 1-hidden layer feedforward neural network without biases. Then for given $p$, both linear programming and $\alpha$-CROWN bounds over the region $\|x\|_p\leq\varepsilon$ almost surely satisfy $B^{\mathrm{LP}} = \Omega(\varepsilon d)$ and $B^{\alpha-\mathrm{CROWN}} = \Omega(\varepsilon d)$.
\end{theorem}

We expect that for neural networks with $L$-hidden layers, the bounds will be even looser, implying that beyond a certain value of $d$, the SDP bound will always be tighter. These results also explain that in the context of adversarial robustness, when $\varepsilon$ is very small, linear bounds may be sufficiently accurate, and computing the expensive SDP bound may not be necessary. When $d$ is moderately small, the branch-and-bound method may work by partitioning the region into sufficiently small subregions of effective size $\varepsilon' < \varepsilon$, and then applying cheap linear bounds. Unfortunately, as $d$ grows, the number of regions grows exponentially, making this approach intractable.

\section{Methods}
In this section we formalize the main task of training a safe network and propose a constrained optimization method to solve it. 

\subsection{Main task}
We aim to train an accurate neural classifier $f_\theta$ on the Adversarial Spheres dataset $\mathcal{D}_2$, described in section~\ref{subsec:setup} with the safety property
\begin{equation}
    \label{eq:task:prop}
    \max_{x\in X} f_\theta(x)\leq 0;\quad X = \{x: \|x\|_2\leq 1\}
\end{equation}
We choose $0$ as the threshold value for the classifier, where negative values of $f_\theta(x)$ correspond to the inner class $\|x\|_2=1$, and positive values correspond to the outer class $\|x\|_2=R$. If this property is satisfied, the classifier is guaranteed to output the correct answer for any point $x$ from the dataset's inner class. Unfortunately, we cannot enforce similar guarantees for the outer class. Due to the convex nature of the SDP bound, any bound on region $X$ is, in fact, a bound on its convex completion. Therefore, if we try to enforce the condition for the outer class $\max_{\|x\|_2 = R} f_\theta(x)\geq 0$, it will also be enforced on the inner class, forcing the classifier to output zero everywhere.

According to the Adversarial Spheres paper~\cite{gilmer2018adversarialspheres}, if we simply train the classifier $f_\theta(x)$ in a supervised manner until it reaches a perfect score, numerous adversarial examples will arise that cause $\max_{x\in X} f_\theta(x)$ to become positive. These examples can be found using a PGD attack, demonstrating that the classifier is provably unsafe. On the other hand, when the classifier does become perfect and safe, we observe that the norms of its weight matrices become so large that any incomplete verifier (including SDP and $\alpha$-CROWN) is unable to verify the network, making the verification process hopeless unless we branch on all neurons.

\subsection{Constrained optimization problem}
We propose to view task as a constrained optimization problem. Specifically, we aim to minimize the classifier loss $\mathcal{L}_c(\theta)$ while satisfying condition~\ref{eq:task:prop}. In its abstract form, this problem can be written as
\begin{align}
    & \min \mathcal{L}_c(\theta) \\
    \label{eq:task:constraint}
    & \max_{x\in X} f_\theta(x)\leq B^{\mathrm{SDP}}(\theta)\leq 0
\end{align}

First, we express this as
\begin{align}
    \max_X \langle C, X\rangle = \min_y\max_X \langle C, X\rangle + y^T(a - \mathcal{A}(X)) = \min_y\max_X a^Ty + \langle C - \mathcal{A}^T(y), X\rangle = a^Ty
\end{align}
as long as $\mathcal{A}^T(y) - C\succeq 0$. Therefore, the dual of the~\ref{eq:sdp:primal} problem is
\begin{align}
    & B^{\mathrm{SDP}}_* = \min_y a^Ty + c \\
    & \mathcal{A}^T(y) - C = S\succeq 0
\end{align}
We assume that the problem is not degenerate and that there is no duality gap. 

The advantage of the dual problem over the primal one is that any valid $(y, S)$ yields an upper bound on $B^{\mathrm{SDP}} = B^{\mathrm{SDP}}_*$, and by~\ref{eq:task:constraint}, on the maximum of $f_\theta(x)$ over $X$. Therefore, we can replace the constraint~\ref{eq:task:constraint} by instead bounding the dual value $a^Ty + c \leq 0$, and write a new non-convex constrained optimization problem
\begin{align}
    & \min \mathcal{L}_c(\theta) \\
    \label{eq:dual:logit-constraint}
    & a^Ty + c \leq 0 \\
    \label{eq:dual:matrix-constraint}
    & \mathcal{A}^T(y) - C = S\succeq 0
\end{align}

\subsection{ADMM scheme}
\label{subsec:admm}
This scheme is inspired by~\cite{Wen2010AlternatingDA} with the addition of optimization over weights $\theta$ and the constraint~\ref{eq:dual:logit-constraint}. We begin by writing the augmented Lagrangian, which includes all constraints
\begin{equation}
    \label{eq:admm:master-lagrangian}
    \mathcal{L} = \mathcal{L}_c(\theta) - x(a^Ty + c + s) - \langle X, \mathcal{A}^T(y) - S - C\rangle + \frac{1}{2\rho}(a^Ty + c + s)^2 + \frac{1}{2\mu}||\mathcal{A}^T(y) - S - C||^2
\end{equation}
Here, we introduce the Lagrange multiplier $x\in\mathbb{R}$ and the slack variable $s\geq 0$ for the constraint~\ref{eq:dual:logit-constraint}, as well as the matrix multiplier $X\succeq 0$ for the dual constraint~\ref{eq:dual:matrix-constraint}. To perform the ADMM steps, we find the optimal $y$ from the KKT conditions
\begin{align}
    & \nabla_y \mathcal{L} = -xa - \mathcal{A}(X) + a(a^Ty + c + s) / \rho + \mathcal{A}(\mathcal{A}^T(y) - S - C) / \mu = 0 \\
    \label{eq:admm:y}
    & y = (\mathcal{A}\mathcal{A}^T / \mu + aa^T / \rho)^{-1}\left[\mathcal{A}(S + C) / \mu + \mathcal{A}(X) - a(c + s) / \rho + ax\right]
\end{align}

The optimal $S$ is found by noting that the lagrangian $\mathcal{L}$ is quadratic in $S$ and completing the square
\begin{equation}
    \mathcal{L} = \frac{1}{2\mu}||S||^2 + \frac{1}{\mu}\langle S, C - \mathcal{A}^T(y)\rangle + \langle X, S\rangle + \ldots = \frac{1}{2\mu}||S - V||^2 + \ldots;\quad V = \mathcal{A}^T(y) - C - \mu X
\end{equation}
Since $S\succeq 0$, the optimal $S$ is found by diagonalizing $V$ as $V = Q\Lambda Q^T$, where $\Lambda$ is a diagonal matrix, and then setting $S = Q\Lambda_+ Q^T$, where $\Lambda_+$ is $\Lambda$ with only positive entries retained. Similarly, the Lagrangian is quadratic in $s$, and the optimal value is found by
\begin{equation}
    \mathcal{L} = \frac{1}{2\rho}\left(s^2 + 2s(a^Ty + c) - 2s\rho x + ...\right) = \frac{1}{2\rho}(s - v)^2 + ...;\quad v = \rho x - (a^Ty + c)
\end{equation}
Since $s\geq 0$, we set $s = \max(0, v)$. 

The entire iteration scheme is as follows
\begin{enumerate}
    \item Set $y^{k+1} = (\mathcal{A}\mathcal{A}^T / \mu + aa^T / \rho)^{-1}\left[\mathcal{A}(S^k + C) / \mu + \mathcal{A}(X^k) - a(c + s^k) / \rho + ax^k\right]$.
    \item Compute $V = \mathcal{A}^T(y^{k+1}) - C - \mu X^k$, diagonalize it, and set $S^{k+1} = Q\Lambda_+Q^T$.
    \item Perform the dual update
    \begin{equation}
        X^{k+1} = X^{k} - \nu\frac{\mathcal{A}^T(y^{k+1}) - S^{k+1} - C}{\mu}
    \end{equation}
    We set $\nu=1.6$, as recommended by~\cite{Wen2010AlternatingDA}.
    \item If the norm $\|\mathcal{A}^T(y^{k+1}) - C - S^{k+1}\|\leq\delta$, where $\delta$ is the set tolerance, continue to the next step. Otherwise, set $s^{k+1} = s^k$, $x^{k+1} = x^k$, $\theta^{k+1} = \theta^k$ and return to step 1.
    \item Compute $v = \rho x^k - (a^Ty^{k+1} + c)$ and set $s^{k+1} = \max(0, v)$.
    \item Perform the dual update 
    \begin{equation}
        x^{k+1} = x^k - \alpha\frac{a^Ty^{k+1} + c + s^{k+1}}{\rho}
    \end{equation}
    \item Finally, minimize $\mathcal{L}$ on $\theta$ by performing a gradient step with learning rate $\eta$. Recompute $\mathcal{A}$, $a$, $C$, and $c$ for the new weights $\theta^{k+1}$ and return to step 1.
\end{enumerate}

Note that the constraint matrices $A_\alpha$ can be efficiently represented as a sum of diagonal and low-rank terms. Since there are $m = O(n)$ constraints, $\mathcal{A}\mathcal{A}^T$ can be efficiently computed in $O(n^3)$. Therefore, entire iteration step has the same time complexity of $O(n^3)$.

\section{Results}
We trained a suite of neural network with different $d$ and $h$ using the ADMM scheme described in subsection~\ref{subsec:admm}. All neural networks had $L=2$ hidden layers. The largest neural network had 40 input neurons, 240 hidden neurons, and 20k parameters. During the experiments, we found that choosing the right hyperparameters, such as the learning rate $\eta$, dual update rate $\alpha$, and stopping tolerance $\delta$, was very important. To ensure a fair stopping criterion for hyperparameter tuning, we fixed the maximum number of allowed ADMM steps for each $d, h$. Further experimental details can be found in appendix~\ref{appendix:experiments:main}.

We report the best accuracy of a trained safe model that satisfies property~\ref{eq:task:prop}. Additionally, we report the total training time as well as the time needed to verify the network with SDP and $\alpha,\beta$-CROWN once it is trained. While the ADMM steps are fast, the obtained bound typically lacks precision. Therefore, we always recompute the SDP bound at the end using the CVXPY package~\cite{diamond2016cvxpy}. The verification time using CVXPY is reported as $t_{\text{SDP}}$. A "Recall" column is added to indicate that the network correctly classifies any point from the inner class, thus satisfying the safety property. 

\begin{table}[H]
    \centering
    \begin{tabular}{|c|c|c|c|c|c|c|c|}
        \hline
        $d$ & $h$ & Recall & Accuracy & Iteration number & $t_{\text{train}}$, s & $t_{\text{SDP}}$, s & $t_{\alpha,\beta-\text{CROWN}}$, s \\ \hline
        5 & 15 & 1.0 & 0.918 & 5000 & 35 & 0.4 & 0.4 \\
        7 & 21 & 1.0 & 0.953 & 10000 & 19 & 0.6 & 14.3 \\
        8 & 24 & 1.0 & 0.934 & 10000 & 22 & 3 & 106 \\
        9 & 27 & 1.0 & 0.930 & 10000 & 22 & 2.4 & 508 \\
        10 & 30 & 1.0 & 0.969 & 10000 & 48 & 2.4 & 7484 \\
        15 & 45 & 1.0 & 0.914 & 10000 & 64 & 13 & - \\
        20 & 60 & 1.0 & 0.930 & 20000 & 206 & 10.4 & - \\
        40 & 120 & 1.0 & 0.918 & 30000 & 1209 & 163 & - \\ \hline      
    \end{tabular}
    \caption{Results of training safe network on $\mathcal{D}_2$}
    \label{tab:results-2}
\end{table}

We also conducted experiments by training neural network in the same fashion on $\mathcal{D}_\infty$ while requiring safety property
\begin{equation}
    \max_{x\in X} f_\theta(x)\leq 0;\quad X = \{x: \|x\|_\infty\leq 1\}
\end{equation}
Unfortunately, resulting accuracy quickly drops as we increase $d$, even when allowing the ADMM algorithm to run much longer than was required for the $\mathcal{D}_2$ case. 

\begin{table}[H]
    \centering
    \begin{tabular}{|c|c|c|c|}
        \hline
        d & h & Accuracy & Iteration number \\ \hline
        2 & 6 & 0.986 & 20000 \\ 
        3 & 9 & 0.881 & 20000 \\ 
        4 & 12 & 0.825 & 20000 \\ 
        5 & 15 & 0.617 & 20000 \\ \hline
    \end{tabular}
    \caption{Results of training safe network on $\mathcal{D}_\infty$}
    \label{tab:results-infty}
\end{table}

\section{Discussion}
From Table~\ref{tab:results-2}, we observe that $\alpha,\beta$-CROWN times out when verifying any of the trained networks with $d > 10$, as the verification time grows exponentially with the network's size, as expected. On the other hand, we successfully trained safe neural networks with the SDP bound up to $d = 40$, with no signs of degrading accuracy, and the total training time grows polynomially.

We suspect that in the case of the $l_\infty$-norm, based on the experimental results of Section~\ref{subsec:gap-exp} and the theoretical result~\ref{thm:theory:sdp-bound}, the SDP bound becomes too imprecise to guide the training effectively. As this gap grows with $d$, which does not happen in the $l_2$-norm case, the final accuracy of the trained network worsens as $d$ increases, as shown in Table~\ref{tab:results-infty}.

The logical continuation of this research is to search for new global bounds that can be used for neural network training in the $l_\infty$-norm case and, at the same time, are computationally efficient enough to be applicable. We highlight works that employ ideas from the sum of squares (SOS) hierarchy~\cite{ahmadi2018dsossdsosoptimizationtractable},~\cite{brown2022unifiedviewsdpbasedneural} and Positivstellensatz~\cite{roebers2021sparsenonsosputinartypepositivstellensatze},~\cite{newton2022sparsepolynomialoptimisationneural}, which naturally extend the SDP framework and provide more accurate bounds. Another direction is to incorporate the SDP bound with a branch-and-bound algorithm. Even though BaB scales poorly with neural network size, for a given network, it may serve as a way to trade off computational complexity with final accuracy by making the SDP bounds tighter and more useful in guiding the network's training.

Finally, it's important to demonstrate the applicability of this method to safe RL control by extending it to simultaneously train neural policies and neural Lyapunov functions~\cite{Chang2022} or reach-avoid supermartingales~\cite{Zikelic2022}, while using SDP bounds to ensure their safety properties.

\section{Conclusion}
This work presents a novel approach for training safe neural networks that satisfy verified properties over large regions in high-dimensional spaces using SDP bounds. By leveraging the power of custom ADMM SDP solver, we are able to efficiently train networks that maintain safety guarantees, even as network dimensions increase, in polynomial time. Our experiments demonstrate the effectiveness of this method for the $l_2$-norm safety specifications, successfully training networks up to $d=40$ without compromising accuracy or safety. This approach has significant potential for future developments in safe reinforcement learning and safety-critical applications where verified robustness is essential.

However, our results also highlight the challenges of dealing with $l_\infty$-norm safety specifications, where the SDP bounds become less effective as the input dimension increases, leading to diminished accuracy. This limitation underscores the need for tighter global bounds that can efficiently handle $l_\infty$-norm cases while remaining computationally feasible, which we aim to explore in future work.

\section{Acknowledgements}
For helpful feedback and discussions we’d like to thank Carson Jones, Justice Sefas, and Henry Sleight. Roman was supported by the Machine Learning Alignment Theory Scholars program and by the Long Term Future Fund.

\printbibliography

\appendix
\section{Experiments details}
\label{appendix:experiments}

\subsection{Gap experiments}
\label{appendix:experiments:gap}
To calculate $B^{\mathrm{PGD}}\leq\max_{x\in X}f_\theta(x)$, we randomly sample a batch of points $\{x_i\}$ of size $b=256$ inside an $l_p$ ball of radius $\varepsilon$ and perform PGD step with respect to empirical maximum as follows
\begin{equation}
    x'_i = \Pi_{p,\varepsilon}\left(x_i + \eta\nabla_{x_i}\max_j f_\theta(x_j)\right)
\end{equation}
where $\Pi_{p,\varepsilon}$ is a projection onto the $l_p$ ball. We continue this process until the empirical maximum $\max_j f_\theta(x_j)$ does not increase for 100 steps, within a tolerance of $10^{-4}$. Finally, we set $B^{\mathrm{PGD}} = \max_j f_\theta(x_j)$. For the gradient ascent steps, we use Adam with a learning rate of $\eta=0.01$ and standard values $\beta_1 = 0.9$, $\beta_2 = 0.999$.

For each experiment, $s=10$ neural networks are created with Xavier initialization. The randomly initialized neural networks are left untrained, and bounds are computed for them. The trained neural networks are optimized using the classification loss $\mathcal{L}_c(\theta)$ until a threshold accuracy of 95\% is reached, or the accuracy does not improve for 100 steps, within a tolerance of $10^{-3}$, using Adam with a learning rate of $\eta=10^{-3}$.

\subsection{Main experiments}
\label{appendix:experiments:main}
All experiments and time measurements were conducted on a single Apple M2 Pro chip. We used Bayesian optimization in logarithmic scale for hyperparameter tuning to find the optimal configuration. The range of variation is shown below.
\begin{table}[H]
    \centering
    \begin{tabular}{|c|c|}
        \hline
        Hyperparameter & Range \\ \hline
        {Matrix penalty $\mu$} & $[0.1, 10]$ \\
        {Logit penalty $\rho$} & $[0.1, 10]$ \\
        {Tolerance $\delta$} & $[10^{-4}, 10^{-2}]$ \\
        {Logit dual step $\alpha$} & $[0.01, 1]$ \\
        {Learning rate $\eta$} & $[10^{-4}, 10^{-2}]$ \\ \hline
    \end{tabular}
    \caption{Hyperparameter search domain}
    \label{tab:hyper}
\end{table}
We report parameters of the best run for a range of models
\begin{table}[H]
    \centering
    \begin{tabular}{|c|c|c|c|c|c|c|c|}
        \hline
        d & h & Iteration number & $\mu$ & $\rho$ & $\delta$ & $\alpha$ & $\eta$ \\ \hline
        5 & 15 & 2000 & 0.65 & 1.21 & $2.4\cdot 10^{-3}$ & 0.046 & $2.65\cdot 10^{-3}$ \\
        10 & 30 & 10000 & 0.23 & 9.9 & $8.1\cdot 10^{-3}$ & 0.28 & $4.3\cdot 10^{-3}$ \\
        20 & 60 & 10000 & 0.45 & 0.90 & $3.7\cdot 10^{-3}$ & 0.020 & $4.7\cdot 10^{-4}$ \\
        40 & 60 & 30000 & 0.87 & 7.9 & $3.3\cdot 10^{-3}$ & 0.066 & $2.1\cdot 10^{-4}$ \\ \hline
    \end{tabular}
    \caption{Best run hyperparameters}
    \label{tab:best-hyper}
\end{table}

Empirically, we have found that the only critically important parameter is the tolerance $\delta$. If it exceeds a certain threshold, the ADMM logit bound $a^Ty + c$ stops converging to the true $B_*^{\mathrm{SDP}}$, resulting in an unsafe neural network. On the other hand, setting $\delta$ too small causes the ADMM scheme to spend more iterations converging only on the variables $(X, y, S)$, without adequately updating the weights $\theta$. It also appears that a large $\rho$, which corresponds to a loose logit penalty in~\ref{eq:admm:master-lagrangian}, is preferable, especially for larger networks. If $\rho$ is too small, then logit is forced to be near zero from the start, making it difficult for the neural network to achieve high accuracy.

Below we show typical training run for a $d=40$, $h=120$ neural network. For each weight update (step 7 of ADMM scheme) we plot: the number of inner matrix iterations (steps 1-3) needed to achieve tolerance $\delta$ and proceed to steps 4-7, the lagrangian~\ref{eq:admm:master-lagrangian}, the logit bound $a^Ty + c$, and the network's accuracy. 
\begin{figure}[H]
    \centering
    \includegraphics[width=0.9\linewidth]{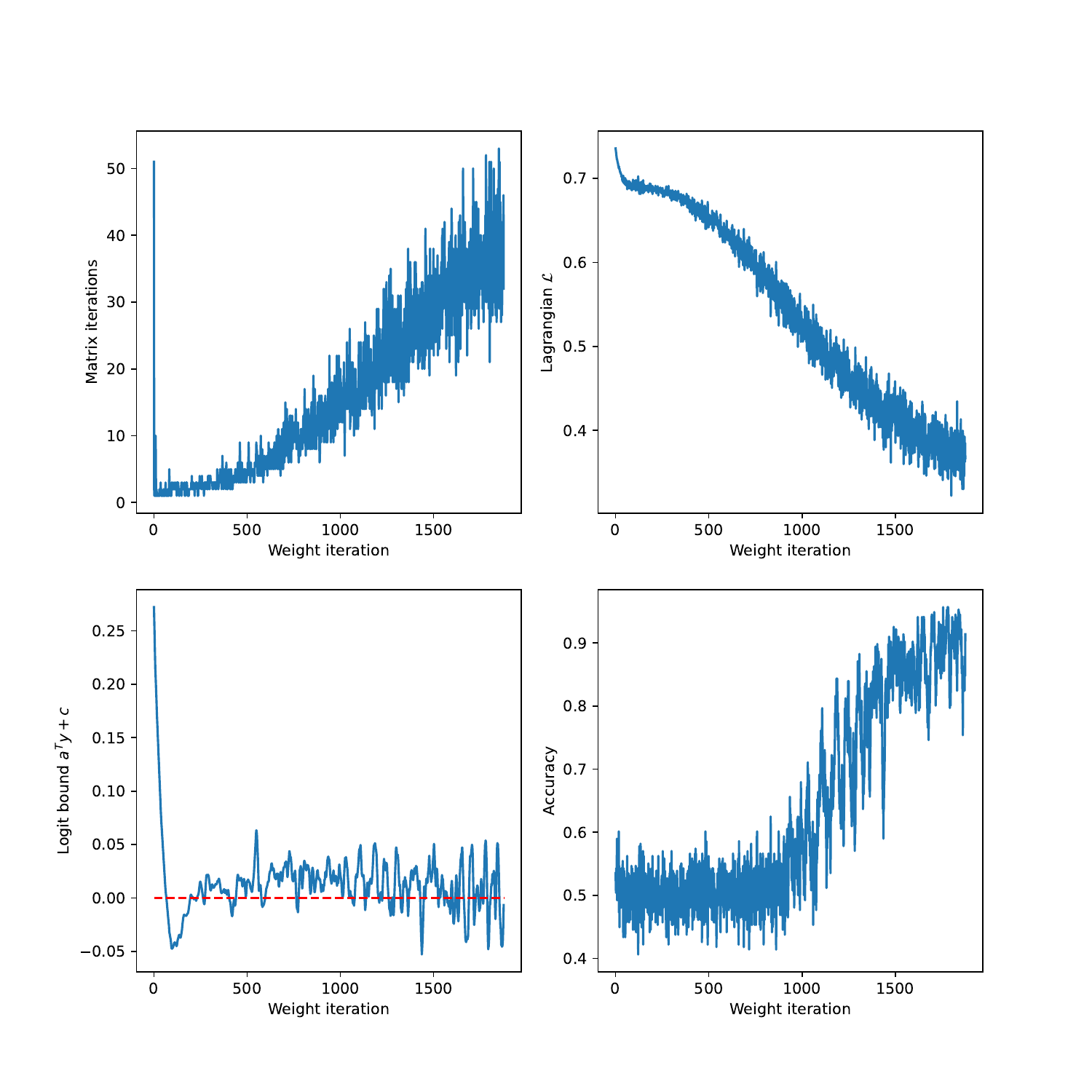}
    \caption{Training trajectory for $d=40$, $h=120$ run}
    \label{fig:train-log}
\end{figure}

\section{Theoretical explanation proof}
\subsection{SDP bound}
\label{appendix:theory:sdp}

We use the following classical lemma without proof
\begin{lemma}
    \label{lemma:theory:block}
    Let $P$ be positive semidefinite block matrix of the form
    \begin{equation}
        P = \begin{bmatrix}
            A & B \\
            B^T & C
        \end{bmatrix}
    \end{equation}
    then $\|B\|_*^2\leq \|A\|_*\cdot\|C\|_*$, where $\|\cdot\|_*$ denotes the matrix nuclear norm.
\end{lemma}

Since the matrix $P$ in~\ref{eq:sdp:master-matrix} is positive semidefinite, any its upper-left corner, including those obtained by permuting rows and columns, is also positive semidefinite. Thus, $\begin{bmatrix}
    P[1] & P[x_L^T] \\
    P[x_L] & P[x_Lx_L^T]
\end{bmatrix}\succeq 0$ and by Lemma~\ref{lemma:theory:block}, $\|P[x_L]\|_*^2 = \|P[x_L]\|_2^2\leq \|P[x_Lx_L^T]\|_*$. This gives us the bound
\begin{equation}
    B^{\mathrm{SDP}} = \max W_{L+1} P[x_L] + b_{L+1}\leq \|W_{L+1}\|_2 \|P[x_L]\|_2 +  \|b_{L+1}\|_2\leq \|W_{L+1}\|_2\sqrt{\|P[x_Lx_L^T]\|_*} + \|b_{L+1}\|_2
\end{equation}
To bound the norm of $P[x_Lx_L^T]$, we will use the following lemma by working backward through the layers.

\begin{lemma}
    \label{lemma:theory:step}
    For any $l=1\ldots L$ we have bound
    \begin{equation}
        \sqrt{\|P[x_lx_l^T]\|_*}\leq||W_l||_2\sqrt{\|P[x_{l-1}x_{l-1}^T]\|_*} + ||b_l||_2
    \end{equation}
\end{lemma}
\begin{proof}
First, we take the trace from equation~\ref{eq:sdp:constr-block}
\begin{equation}
    \|P[x_lx_l^T]\|_* = \tr P[x_lx_l^T] = \tr W_l P[x_{l-1}x_l^T] + b_l^T P[x_l]\leq \|W_l\|_2 \tr P[x_{l-1}x_l^T] + \|b_l\|_2\sqrt{\tr P[x_lx_l^T]}
\end{equation}
Then we apply Lemma~\ref{lemma:theory:block} to the submatrix $\begin{bmatrix}
    P[x_{l-1}x_{l-1}^T] & P[x_{l-1}x_l^T] \\
    P[x_lx_{l-1}^T] & P[x_lx_l^T]
\end{bmatrix}\succeq 0$ which gives $\|P[x_lx_{l-1}^T]\|_*^2\leq \|P[x_{l-1}x_{l-1}^T]\|_*\|P[x_lx_l^T]\|_*$
As a result,
\begin{equation}
    \|P[x_lx_l^T]\|_*\leq \left(\|W_l\|_2\sqrt{\|P[x_{l-1}x_{l-1}^T]\|_*} + \|b_l\|_2\right)\sqrt{\|P[x_lx_l^T]\|_*}
\end{equation}
which readily leads to the desired conclusion.
\end{proof}

Now we can prove 
\begin{proof}[Proof of Theorem~\ref{thm:theory:sdp-bound}]
By recursively applying Lemma~\ref{lemma:theory:step}, we obtain
\begin{equation}
    B^{\mathrm{SDP}}\leq\prod_{l=1}^{L+1}\|W_l\|_2\cdot \sqrt{\|P[x_0x_0^T]\|_*} + \sum_{l=1}^{L+1}\|b_l\|_2\prod_{m=l+1}^{L+1}\|W_l\|_2
\end{equation}
The norm of the input block is constrained by the input region. If $p=2$, $\|P[x_0x_0^T]\|_* = \tr P[x_0x_0^T]\leq\varepsilon^2$. If $p=\infty$, $\|P[x_0x_0^T]\|_*\leq d\varepsilon^2$. 

Under Xavier initialization, each entry of $b\times a$ weight matrix and each element of a $b$-dimensional bias vector is initialized independently from $\mathrm{Uniform}[-1/\sqrt{a},1/\sqrt{a}]$. Then, the bias norm $\|\cdot\|_2$ concentrates around $\sqrt{b/3a}$, while the matrix norm $\|\cdot\|_2$ concentrates around $(\sqrt{a}+\sqrt{b})/\sqrt{3a}$, by Bai-Yin's law~\cite{BaiYin1993}. Based on the network's architecture, we can compute the limiting values of each $\|W_l\|_2$ and $\|b_l\|_2$. Summing up, we almost surely obtain, as $d, h\to\infty$
\begin{equation}
    B^{\mathrm{SDP}}\leq (1 + o(1))\left[\frac{2^{L-1}}{\sqrt{3^{L+1}}}\left(1 +\sqrt{\frac{h}{d}}\right) \sqrt{\|P[x_0x_0^T]\|_*} + \sqrt{\frac{h}{3d}}\frac{2^{L-1}}{\sqrt{3^L}} + \frac{(2/\sqrt{3})^{L-1}-1}{2\sqrt{3} - 3}\right]
\end{equation}
\end{proof}

\subsection{LP and $\alpha$-CROWN bound}
\label{appendix:theory:lp-crown}

\begin{proof}[Proof of Theorem~\ref{thm:theory:alpha-bound}]
   First we prove a lower bound for $\alpha$-CROWN, followed by the linear programming (LP) bound. We start by unfolding the output computation procedure~\ref{eq:setup:input}-~\ref{eq:setup:output} for an $L = 1$ network without biases, and then calculate the intermediate bounds starting from the input, as decribed in~\cite{xu2021fastcompleteenablingcomplete}.
   \begin{align}
       y = W_1 x\\
       x_1 = \sigma(y) \\
       f_\theta(x) = W_2 x_1
   \end{align}
   The input constraint is $x\in X = \{x:||x||_p\leq\varepsilon\}$, meaning that each coordinate is bounded as $-\varepsilon\leq x^j\leq \varepsilon$. Then $y$ is bounded as
   \begin{equation}
       -\varepsilon\sum_i |W_1^{ji}|\leq y^j \leq \varepsilon\sum_i |W_1^{ji}|
   \end{equation}

   Under Xavier initialization, $W_1^{ji}$ is sampled from $\mathrm{Uniform}[-1/\sqrt{d}, 1/\sqrt{d}]$, and $\sum_i |W_1^{ji}|$ concentrates around $\sqrt{d}/2$. Neuron $j$ is unstable and $\alpha$-CROWN bound looks like
   \begin{equation}
       \label{eq:crown:inter}
       \alpha^j y^j\leq x_1^j \leq \frac{1}{2}y^j + \frac{\varepsilon}{2}\sum_i |W_1^{ji}|
   \end{equation}
   The upper bound on $f_\theta(x)$ is
   \begin{equation}
       f_\theta(x)\leq \sum_j W_2^j y^j\left(\frac{1}{2} \mathbf{1}[W_2^j\geq 0] + \alpha^j\mathbf{1}[W_2^j\leq 0]\right) + \frac{\varepsilon}{2}\sum_j W_2^j \mathbf{1}[W_2^j\geq 0]\sum_i |W_1^{ji}|
   \end{equation}
   
   Whatever the linear relationship between $y$ and $x$, the linear perturbation bound over ball $X$ cannot be smaller than the bias term, we have
   \begin{equation}
       \label{eq:crown:bound}
       \max_{x\in X} f_\theta(x) \leq B^{\alpha-\mathrm{CROWN}};\quad B = \frac{\varepsilon}{2}\sum_j W_2^j \mathbf{1}[W_2^j\geq 0]\sum_i |W_1^{ji}|\leq B^{\alpha-\mathrm{CROWN}}
   \end{equation}

   The linear programming problem for such a neural network is formulated as 
   \begin{align}
       & \max W_2 x_1 \\
       & -\varepsilon\leq x \leq \varepsilon, \\
       & x_1 \geq 0, x_1 \geq W_1 x, \\
       & x_1^j \leq \frac{1}{2} y^j + \frac{\varepsilon}{2}\sum_i |W_1^{ji}|
   \end{align}
   where the last bound is the same as bound~\ref{eq:crown:inter}, with the only difference being that $\alpha$-CROWN chooses one lower bound instead of two conditions to propagate it efficiently. We find a feasible solution by setting 
   \begin{equation}
       x = 0\Rightarrow y = 0;\quad x_1^j = \frac{\varepsilon}{2}\mathbf{1}[W_2^j\geq 0]\sum_i |W_1^{ji}|
   \end{equation}
   resulting in the same bound $B$ as in~\ref{eq:crown:bound}. When $d, h\to\infty$, $B$ concentrates around $\varepsilon\sqrt{dh}/16$, giving us the result.
   
\end{proof}

\end{document}